\pdfoutput=1
\documentclass{article}

\usepackage[bw, framed, autolinebreaks, useliterate]{mcode}
\usepackage{times}
\usepackage{graphicx} 
\usepackage{subfigure, subfloat}

\usepackage{xr}
\usepackage{stmaryrd}
\usepackage{hyperref}
\usepackage{amssymb}
\usepackage{multirow}
\usepackage{amsmath}
\usepackage{booktabs}
\usepackage{mathtools}

\usepackage{natbib}

\usepackage{algorithm}
\usepackage{algorithmic}

\usepackage[accepted]{icml2017}

\usepackage{amsthm}
\theoremstyle{plain}
\newtheorem{theorem}{\textbf{Theorem}}

\newtheorem{lemma}{Lemma}
\newtheorem{corollary}{\textbf{Corollary}}
\newtheorem{proposition}{\textbf{Proposition}}
\newcommand{\norm}[1]{\left\lVert#1\right\rVert} 
\usepackage[toc,page]{appendix}

\icmltitlerunning{Efficient Orthogonal Parametrisation of Recurrent Neural Networks}

\begin{document} 
	
	\twocolumn[
	\icmltitle{Efficient Orthogonal Parametrisation of Recurrent Neural Networks \\ Using Householder Reflections}

\icmlsetsymbol{equal}{*}

\begin{icmlauthorlist}
	\icmlauthor{Zakaria Mhammedi}{goo,eg}
	\icmlauthor{Andrew Hellicar}{eg}
	\icmlauthor{Ashfaqur Rahman}{eg}
	\icmlauthor{James Bailey}{goo}
\end{icmlauthorlist}

\icmlaffiliation{eg}{Data61, CSIRO, Australia}
\icmlaffiliation{goo}{The University of Melbourne, Parkville, Australia}

\icmlcorrespondingauthor{Zakaria Mhammedi}{zak.mhammedi@data61.csiro.au}




	\vskip 0.3in
	]
	
	\printAffiliationsAndNotice{} %

	\begin{abstract}
	The problem of learning long-term dependencies in sequences using Recurrent Neural Networks (RNNs) is still a major challenge. Recent methods have been suggested to solve this problem by constraining the transition matrix to be unitary during training which ensures that its norm is equal to one and prevents exploding gradients. These methods either have limited expressiveness or scale poorly with the size of the network when compared with the simple RNN case, especially when using stochastic gradient descent with a small mini-batch size. Our contributions are as follows; we first show that constraining the transition matrix to be unitary is a special case of an orthogonal constraint. Then we present a new parametrisation of the transition matrix which allows efficient training of an RNN while ensuring that the matrix is always orthogonal.  
	Our results show that the orthogonal constraint on the transition matrix applied through our parametrisation gives similar benefits to the unitary constraint, without the time complexity limitations. 
	\end{abstract}

	\section{\textbf{Introduction}}
	
	Recurrent Neural Networks (RNNs) have been successfully used in many applications involving time series. This is because RNNs are well suited for sequential data as they process inputs one element at a time and store relevant information in their hidden state. In practice, however, training simple RNNs (sRNN) can be challenging due to the problem of exploding and vanishing gradients \cite{hochreiter2001gradient}. It has been shown that exploding gradients can occur when the transition matrix of an RNN has a spectral norm larger than one \cite{glorot2010understanding}. This results in an error surface, associated with some objective function, having very steep walls \cite{pascanu2013difficulty}. On the other hand, when the spectral norm of the transition matrix is less than one, the information at one time step tend to vanish quickly after a few time steps. This makes it challenging to learn long-term dependencies in sequential data. 
	
	Different methods have been suggested to solve either the vanishing or exploding gradient problem. The LSTM has been specifically designed to help with the vanishing gradient \cite{Hochreiter1997}. This is achieved by using gate vectors which allow a linear flow of information through the hidden state. 
	However, the LSTM does not directly address the exploding gradient problem. One approach to solving this issue is to clip the gradients \cite{mikolov2012statistical} when their norm exceeds some threshold value. However, this adds an extra hyperparameter to the model. Furthermore, if exploding gradients can occur within some parameter search space, the associated error surface will still have steep walls. This can make training challenging even with gradient clipping. 
	
	Another way to approach this problem is to improve the shape of the error surface directly by making it smoother, which can be achieved by constraining the spectral norm of the transition matrix to be less than or equal to one. However, a value of exactly one is best for the vanishing gradient problem. A good choice of the activation function between hidden states is also crucial in this case. These ideas have been investigated in recent works. In particular, the unitary RNN \cite{arjovsky2015unitary} uses a special parametrisation to constrain the transition matrix to be unitary, and hence, of norm one. This parametrisation and other similar ones \cite{hyland2016learning,wisdom2016full} have some advantages and drawbacks which we will discuss in more details in the next section. 
	
	The main contributions of this work are as follows: 
	\begin{itemize}
		\item We first show that constraining the search space of the transition matrix of an RNN to the set of unitary matrices $\mathbf{U}(n)$ is equivalent to limiting the search space to a subset of $\mathbf{O}(2n)$ ($\mathbf{O}(2n)$ is the set of $2n \times 2n$ orthogonal matrices) of a new RNN with twice the hidden size. This suggests that it may not be necessary to work with complex matrices.  
		
		\item We present a simple way to parametrise orthogonal transition matrices of RNNs using Householder matrices, and we derive the expressions of the back-propagated gradients with respect to the new parametrisation. This new parametrisation can also be used in other deep architectures.    
		
		\item We develop an algorithm to compute the back-propagated gradients efficiently. Using this algorithm, we show that the worst case time complexity of one gradient step is of the same order as that of the sRNN.  
		
	\end{itemize}

	\section{\textbf{Related Work}}
	\label{se1}
	Throughout this work we will refer to elements of the following sRNN architecture.
	\begin{align}
	\label{eq15}
	h^{(t)} &= \phi(W h^{(t-1)} + V x^{(t)}),\\
	o^{(t)} &= Y h^{(t)},
	\end{align}
	where $W$, $V$ and $Y$ are the hidden-to-hidden, input-to-hidden, and hidden-to-output weight matrices. $h^{(t-1)}$ and $h^{(t)}$ are the hidden vectors at time steps $t-1$ and $t$ respectively. Finally, $\phi$ is a non-linear activation function. We have omitted the bias terms for simplicity.
	
	Recent research explored how the initialisation of the transition matrix $W$ influences training and the ability to learn long-term dependencies. In particular, initialisation with the identity or an orthogonal matrix can greatly improve performance \cite{le2015simple}. In addition to these initialisation methods, one study also considered removing the non-linearity between the hidden-to-hidden connections \cite{henaff2016orthogonal}, i.e. the term $W h^{(t-1)}$ in Equation \eqref{eq15} is outside the activation function $\phi$. This method showed good results when compared to the LSTM on pathological problems exhibiting long-term dependencies.  
	
	After training a model for a few iterations using gradient descent, nothing guarantees that the initial structures of the transition matrix will be held. In fact, its spectral norm can deviate from one, and exploding and vanishing gradients can be a problem again. It is possible to constrain the transition matrix to be orthogonal during training using special parametrisations \cite{arjovsky2015unitary, hyland2016learning}, which ensure that its spectral norm is always equal to one. 
	The unitary RNN (uRNN) \cite{arjovsky2015unitary} is one example where the hidden matrix $W \in \mathbb{C}^{n \times n}$ is the product of elementary matrices, consisting of reflection, diagonal, and Fourier transform matrices. When the size of hidden layer is equal to $n$, the transition matrix has a total of only $7n$ parameters. Another advantage of this parametrisation is computational efficiency - the matrix-vector product $Wv$, for some vector $v$, can be calculated in time complexity $\mathcal{O}(n \log n)$. However, it has been shown that this parametrisation does not allow the transition matrix to span the full unitary group \cite{wisdom2016full} when the size of the hidden layer is greater than 7. This may limit the expressiveness of the model. 
	
	Another interesting parametrisation \cite{hyland2016learning} has been suggested which takes advantage of the algebraic properties of the unitary group $\mathbf{U}(n)$. The idea is to use the corresponding matrix Lie algebra $\mathbf{u}(n)$ of skew hermitian matrices. In particular, the transition matrix can be written as $W = \exp\left[\sum_{i=1}^{n^2} \lambda_i T_i   \right]$,
	where $\exp$ is the exponential matrix map and $\{T_i\}^{n^2}_{i=1}$ are predefined $n\times n$ matrices forming a bases of the Lie algebra $\mathbf{u}(n)$. The learning parameters are the weights $\{\lambda_i\}$. The fact that the matrix Lie algebra $\mathbf{u}(n)$ is closed and connected ensures that the exponential mapping from $\mathbf{u}(n)$ to  $\mathbf{U}(n)$ is surjective. Therefore, with this parametrisation the search space of the transition matrix spans the whole unitary group. This is one advantage over the original unitary parametrisation \cite{arjovsky2015unitary}. However, the cost of computing the matrix exponential to get $W$ is $\mathcal{O}(n^3)$, where $n$ is the dimension of the hidden state. . 
	
	Another method \cite{wisdom2016full} performs optimisation directly of the Stiefel manifold using the Cayley transformation. The corresponding model was called full-capacity unitary RNN. Using this approach, the transition matrix can span the full set of unitary matrices. However, this method involves a matrix inverse as well as matrix-matrix products which have time complexity $\mathcal{O}(n^3)$. This can be problematic for large neural networks when using stochastic gradient descent with a small mini-batch size. 
	
	A more recent study \cite{vorontsov2017orthogonality} investigated the effect of soft versus hard orthogonal constraints on the performance of RNNs. The soft constraint was applied by specifying an allowable range for the maximum singular value of the transition matrix. To this end, the transition matrix was factorised as $W= USV'$, where $U$ and $V$ are orthogonal matrices and $S$ is a diagonal matrix containing the singular values of $W$. A soft orthogonal constraint consists of specifying small allowable intervals around 1 for the diagonal elements of $S$. Similarly to \cite{wisdom2016full}, the matrices $U$ and $V$ were updated at each training iteration using the Cayley transformation, which involves a matrix inverse, to ensure that they remain orthogonal. 
	
	All the methods discussed above, except for the original unitary RNN, involve a step that requires at least a $\mathcal{O}(n^3)$ time complexity. All of them, except for one, require the use of complex matrices. Table \ref{t4} summarises the time complexities of various methods, including our approach, for one stochastic gradient step. In the next section, we show that imposing a unitary constraint on a transition matrix $W \in \mathbb{C}^{n \times n}$ is equivalent to imposing a special orthogonal constraint on a new RNN with twice the hidden size. Furthermore, since the norm of orthogonal matrices is also always one, using the latter has the same theoretical benefits as using unitary matrices when it comes to the exploding gradient problem. 
	\begin{table*}
		\centering
		\begin{tabular}{cccc}
			\hline
			Methods    & Constraint on the&   Time complexity of one & Search space of the   \\
			& transition matrix & online gradient step & transition matrix \\
			\hline
			uRNN \cite{arjovsky2015unitary} & $\norm{W}=1 $& $\mathcal{O}(Tn\log(n))$ & A subset of  $\mathbf{U}(n)$   \\
			& & & when $n > 7$\\
			\hline
			Full-capacity uRNN  & $\norm{W}=1 $& $\mathcal{O}(T n^2 + n^3)$ & The full $\mathbf{U}$(n) set  \\
			\cite{wisdom2016full}  & &  &\\
			\hline 
			Unitary RNN  &$ \norm{W}= 1$ & $\mathcal{O}(T n^2 + n^3)$ & The full $\mathbf{U}$(n) set \\
			\cite{hyland2016learning} & & &\\
			\hline
			\hline
			oRNN & $\norm{W}= 1$ &$\mathcal{O}(T nm)$ &  The full $\mathbf{O}$(n) set\\
			(Our approach) & &   where $m \leq n$ &  when $m=n$ \\
			\hline
		\end{tabular}
		\caption{Table showing the time complexities associated with one stochastic gradient step (mini-batch size =1) for different methods, when the size of the hidden layer of an RNN is $n$ and the input sequence has length $T$. }
			\label{t4}
	\end{table*}

	\section{\textbf{Complex unitary versus orthogonal}}
	\label{se2}
	
	We can show that when the transition matrix $W\in \mathbb{C}^{n \times n}$ of an RNN is unitary, there exists an equivalent representation of this RNN involving an orthogonal matrix $\hat{W} \in \mathbb{R}^{2n \times 2n}$.
	
	In fact, consider a complex unitary transition matrix $W = A + i B \in \mathbb{C}^{n \times n}$, where A and B are now real-valued matrices in $\mathbb{R}^{n \times n}$. We also define the following new variables
	\begin{align*}
	\begin{array}{cc}
	\forall t, \hat{h}^{(t)} = \begin{bmatrix}
	\Re \left(h^{(t)}\right)\\
	\Im \left(h^{(t)}\right) 
	\end{bmatrix}, 
	\hat{V}  = \begin{bmatrix}
	\Re \left(V\right)\\  
	\Im \left(V\right) 
	\end{bmatrix},   
	\hat{W} = \begin{bmatrix}
	A  & -B\\
	B & A
	\end{bmatrix} ,
	\end{array}
	\end{align*}
	where $\Re$ and $\Im$ denote the real and imaginary parts of a complex number. Note that $ \hat{h}^{(t)} \in \mathbb{R}^{2n}$, $\hat{W} \in \mathbb{R}^{2n \times 2n}$, and $\hat{V}  \in \mathbb{R}^{2n \times n_x}$, where $n_x$ is the dimension of the input vector $x^{(t)}$ in Equation \eqref{eq15}.
	
	Assuming that the activation function $\phi$ applies to the real and imaginary parts separately, it is easy to show that the update equation of the complex hidden state $h^{(t)}$ of the unitary RNN has the following real space representation
	\begin{align}
	\label{eq:9}
	\hat{h}^{(t)} &= \phi(\hat{W} \hat{h}^{(t-1)} + \hat{V} x^{(t)}).
	\end{align}
	Even when the activation function $\phi$ does not apply to the real and imaginary parts separately, it is still possible to find an equivalent representation in the real space. Consider the activation function proposed by \cite{arjovsky2015unitary} 
	\begin{align}
	\label{eq104}
	\sigma_{\mbox{modRelU}} (z) = \left\{  \begin{matrix} (|z| +b ) \frac{z}{|z|}, & \mbox{if}\; |z| + b >0 \\ 0, & \mbox{otherwise}\; \end{matrix} \right.
	\end{align} 
	where $b$ is a bias vector.
	For a hidden state $\hat{h} \in \mathbb{R}^{2n} $, the equivalent activation function in the real space representation is given by  
	\begin{align*}
	\left[\hat{\phi}(a)\right]_i = \left\{  \begin{matrix} \frac{\sqrt{a_i^2 + a_{k_i}^2} + b_{\tilde{k}_i}}{\sqrt{a_i^2 + a_{k_i}^2}}  a_i, & \mbox{if}\; \sqrt{a_i^2 + a_{k_i}^2} + b_{\tilde{k}_i} >0 \\ 0, & \mbox{otherwise} \end{matrix} \right.
	\end{align*} 
	where $k_i =  \left((i + n) \bmod 2n\right) $ and $\tilde{k}_i=\left(i \bmod n\right)$ for all $i \in \{1,\dots, 2n \}$. The activation function $\hat{\phi}$ is no longer applied to hidden units independently.
	
	Now we will show that the matrix $\hat{W}$ is orthogonal.
	By definition of a unitary matrix, we have $W^{} W^* = I $ 
	where the ${}^*$  represents the conjugate transpose. This implies that 
	$ A A' + B B' = I$ and $BA' - AB' = 0$. And since we have 
	\begin{align}
	\hat{W} \hat{W}'   = \begin{bmatrix}
	A A' + B B' & AB' - BA' \\
	BA' - AB' & A A' + B B'
	\end{bmatrix},
		\end{align} 
	it follows  that $\hat{W} \hat{W}' = I$. Also note that $\hat{W}$ has a special structure - it is a block-matrix. 
	
	The discussion above shows that using a complex, unitary transition matrix in $\mathbb{C}^{n \times n}$ is equivalent to using an orthogonal matrix, belonging to a subset of $\mathbf{O}(2n)$, in a new RNN with twice the hidden size. This is why in this work we focus mainly on parametrising orthogonal matrices. 
	
	\section{\textbf{Parametrisation of the transition matrix}}
	Before discussing the details of our parametrisation, we first introduce a few notations. For $n, k\in \mathbb{N}$ and $ 2 \leq k \leq n $, let $\mathcal{H}_k: \mathbb{R}^{k} \rightarrow \mathbb{R}^{n \times n}$ be defined as
		\begin{equation}
		\label{eq:19}
		\begin{aligned}
		\mathcal{H}_k(\mathbf{u}) =  \begin{bmatrix} I_{n-k} & 0 \\  0 &  I_k - 2 \frac{\mathbf{u} \mathbf{u}'}{\norm{\mathbf{u}}^2}   \end{bmatrix},
		\end{aligned}
		\end{equation}
		where $I_k$ denotes the $k$-dimensional identity matrix. For $\mathbf{u} \in \mathbb{R}^{k}$, $\mathcal{H}_k(\mathbf{u})$ is the \textit{Householder Matrix} in $\mathbf{O}(n)$ representing the reflection about the hyperplane orthogonal to the vector $(\mathbf{0}'_{n-k},\mathbf{u}')'\in \mathbb{R}^n$ and passing through the origin, where $\mathbf{0}_{n-k}$ denotes the zero vector in $\mathbb{R}^{n-k}$. 
		
		 We also define the mapping $\mathcal{H}_1: \mathbb{R} \rightarrow \mathbb{R}^{n\times n}$ as 
		\begin{equation}
		\label{eq:18}
		\begin{aligned}
		\mathcal{H}_1(u)=\begin{bmatrix} I_{n-1} & 0 \\  0 & u \end{bmatrix}.
		\end{aligned}
		\end{equation}
		
	 Note that $\mathcal{H}_{1}(u)$ is not necessarily a Householder reflection. However, when $u \in \{1, -1\}$, $\mathcal{H}_1(u)$ is orthogonal. 
	
	Finally, for $n,k\in \mathbb{N}$ and $1 \leq k \leq n$, we define  
	\begin{align*}
	\mathcal{M}_k: \mathbb{R}^k\times \dots \times \mathbb{R}^n &\rightarrow \mathbb{R}^{n \times n}\\
	(\mathbf{u}_k, \dots, \mathbf{u}_n ) &\mapsto  \mathcal{H}_n(\mathbf{u}_n) \dots \mathcal{H}_{k}(\mathbf{u}_{k}).
	\end{align*}

	 We propose to parametrise the transition matrix $W$ of an RNN using the mappings $\{\mathcal{M}_k\}$. When using $m$ \textit{reflection} vectors $\{\mathbf{u}_i\}$, the parametrisation can be expressed as 
	 \begin{align}
	 W &= \mathcal{M}_{n-m+1}(\mathbf{u}_{n-m+1}, \dots, \mathbf{u}_n) \nonumber \\ 
	 &=  \mathcal{H}_{n}(\mathbf{u}_n)  \dots \mathcal{H}_{n-m+1}(\mathbf{u}_{n-m+1}),   \label{eq12}
	 \end{align}
	 where $\mathbf{u}_i \in \mathbb{R}^{i}$ for $i \in \{n-m+1,\dots, n\}$.
	  
	  For the particular case where $m=n$ in the above parametrisation, we have the following result.
	\begin{theorem}
		\label{t2} 
		The image of $\mathcal{M}_1$ includes the set of all $n\times n$ orthogonal matrices, i.e. $\mathbf{O}(n) \subset \mathcal{M}_1 [\mathbb{R}\times \dots \times \mathbb{R}^n]$.
	\end{theorem}
	
		Note that Theorem \ref{t2} would not be valid if $\mathcal{H}_1(\cdot)$ was a standard Householder reflection. In fact, in the two-dimensional case, for instance, the matrix $\left(\begin{smallmatrix} 1 & 0 \\ 0 & -1 \end{smallmatrix}\right)$ cannot be expressed as the product of exactly two standard Householder matrices.
	
	The parametrisation in \eqref{eq12} has the following advantages:
			\begin{enumerate}
		\item The parametrisation is smooth\footnote{except on a subset of zero Lebesgue measure.}, which is convenient for training with gradient descent. It is also flexible - a good trade-off between expressiveness and speed can be found by tuning the number of reflection vectors.  
		\item The time and space complexities involved in one gradient calculation are, in the worst case, the same as that of the sRNN with the same number of hidden units. This is discussed in the following subsections.
		\item When $m < n$, the matrix $W$ is always orthogonal, as long as the reflection vectors are nonzero. For $m=n$, the only additional requirement for $W$ to be orthogonal is that $u_1 \in \{-1, 1\}$. 
		\item When $m=n$, the transition matrix can span the whole set of $n\times n$ orthogonal matrices. In this case, the total number of parameters needed for $W$ is $n(n+1)/2$. This results in only $n$ redundant parameters since the orthogonal set $\mathbf{O}(n)$ is a $n(n-1)/2$ manifold. 
	\end{enumerate}

	\subsection{\textbf{Back-propagation algorithm}}
	Let $\mathbf{u}_i \in \mathbb{R}^{i}$. Let $U\coloneqq(\mathbf{u}_n| \dots| \mathbf{u}_{n-m+1}) \in \mathbb{R}^{n \times m}$ be the parameter matrix constructed from the reflection vectors $\{\mathbf{u}_i\}$. In particular, the $j$-th column of $U$ can be expressed using the zero vector $\mathbf{0}_{j-1}\in \mathbb{R}^{j-1}$ as 
	\begin{align}
	U_{*,j} = \left[ \begin{array}{c} \mathbf{0}_{j-1} \\ \mathbf{u}_{n-j+1} \end{array} \right] \in \mathbb{R}^n, \quad 1 \leq j \leq m.   \label{eq101}
	\end{align} 
	
	Let $\mathcal{L}$ be a scalar loss function and $C^{(t)} \coloneqq W h^{(t-1)}$, where $W$ is constructed using the $\{\mathbf{u}_i\}$ vectors following Equation \eqref{eq12}. In order to back-propagate the gradients through time, we need to compute the following partial derivatives
	\begin{align}
	\frac{\partial \mathcal{L}}{\partial U^{(t)}} &\coloneqq \left[\frac{\partial C^{(t)}}{\partial U}\right]' \frac{\partial \mathcal{L}}{\partial C^{(t)}},\label{e103}\\
\frac{\partial \mathcal{L}}{\partial h^{(t-1)}}   &= \left[\frac{\partial C^{(t)}}{\partial h^{(t-1)}}\right]' \frac{\partial \mathcal{L}}{\partial C^{(t)}}\label{e6},
	\end{align} 
	at each time step $t$. Note that in Equation \eqref{e103}  $h^{(t-1)}$ is taken as a constant with respect to $U$. Furthermore, we have $\frac{\partial \mathcal{L}}{\partial U} = \sum_{t=1}^{T} \frac{\partial \mathcal{L}}{\partial U^{(t)}}$, where $T$ is the length of the input sequence. The gradient flow through the RNN at time step $t$ is shown in Figure \ref{f6}.
	\begin{figure}[h!]
		\centering
		\includegraphics[trim=1cm 1.3cm 1cm 1.3cm, clip,  width=1\linewidth]{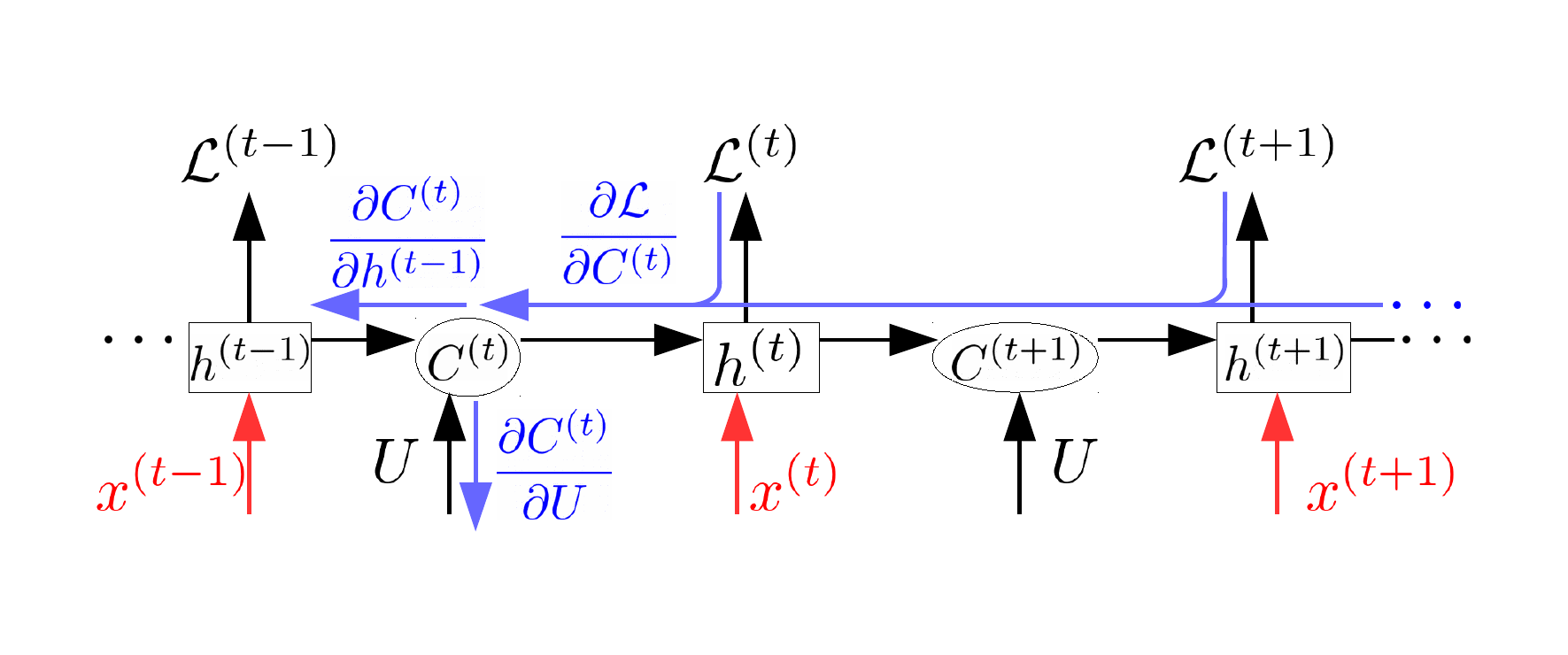}
		\caption{Gradient flow through the RNN at time step $t$. Note that we have $\mathcal{L} = \sum_{t=1}^{T} \mathcal{L}_t(o_t, y_t)$, where $\{y_t\}$ are target outputs.}
		\label{f6}
	\end{figure}
	
	Before describing the algorithm to compute the back-propagated gradients $\frac{\partial \mathcal{L}}{\partial U^{(t)}}$ and $\frac{\partial \mathcal{L}}{\partial h^{(t-1)}}$, we first derive their expressions as a function of $U$, $h^{(t-1)}$ and $ \frac{\partial \mathcal{L}}{\partial C^{(t)}}$ using the \textit{compact WY representation} \cite{joffrain2006accumulating} of the product of Householder reflections. 
	\begin{proposition}
		\label{p1}
		Let $n, m \in \mathbb{N}$ s.t. $m \leq n - 1$. Let $\mathbf{u}_i \in  \mathbb{R}^{i}$ and $U= (\mathbf{u}_n| \dots| \mathbf{u}_{n - m +1})$ be the matrix defined in Equation \eqref{eq101}. We have 
		\begin{align}
		& T \coloneqq \mbox{striu}(U'U) + \frac{1}{2} \mbox{diag}(U'U), \label{eq8} \\
		& \mathcal{H}_n(\mathbf{u}_n) \dots \mathcal{H}_{n-m+1}(\mathbf{u}_{n-m+1})=  I  - UT^{-1} U', \label{eq100}
		\end{align}
		where $\mbox{striu}(U'U)$, and $\mbox{diag}(U'U)$ represent the strictly upper part and the diagonal of the matrix $U'U$, respectively. 
	\end{proposition}
	Equation \eqref{eq100} is the compact WY representation of the product of Householder reflections. For the particular case where $m=n$, the RHS of Equation \eqref{eq100} should be replaced by $\left(I  - U T^{-1} U'\right) \mathcal{H}_1(u_1)$, where $\mathcal{H}_1$ is defined in \eqref{eq:18} and $U = (\mathbf{u}_n| \dots| \mathbf{u}_2)$. 
	
	The following theorem gives the expressions of the gradients $\frac{\partial \mathcal{L}}{\partial U^{(t)}}$ and $\frac{\partial \mathcal{L}}{\partial h^{(t-1)}}$ when $m \leq n-1$ and $h=h^{(t-1)}$.
	\begin{theorem}
		\label{th2}
		Let $n,m \in \mathbb{N}$ s.t. $m \leq n - 1$. Let $U\in \mathbb{R}^{n \times m}$, $h\in \mathbb{R}^n$, and $C=(I - U T^{-1} U')h$, where $T$ is defined in Equation \eqref{eq8}. If $\mathcal{L}$ is a scalar loss function which depends on $C$, then we have
		\begin{align}
		\label{eq5}
	\frac{\partial \mathcal{L}}{\partial U}	 =& U [ (\tilde{h} \tilde{C}') \circ  B' + (\tilde{C} \tilde{h}') \circ  B ]  - \dfrac{\partial \mathcal{L}}{\partial C} \tilde{h}'   - h \tilde{C}',   \\
		\dfrac{\partial \mathcal{L}}{\partial h} = &\frac{\partial \mathcal{L}}{\partial C} - U \tilde{C}, \label{eq6}
		\end{align}
		where $\tilde{h}  = T^{-1}U'h$, $\tilde{C} = (T')^{-1} U'\frac{\partial \mathcal{L}}{\partial C}$, and  $B = \mbox{striu}(J) + \frac{1}{2} I$, with $J$ being the $m\times m$ matrix of all ones and $\circ$ the Hadamard product. 
	\end{theorem}

	The proof of Equations \eqref{eq5} and \eqref{eq6} is provided in Appendix \ref{A}. Based on Theorem \ref{th2}, Algorithm \ref{alg} performs the one-step forward-propagation (FP) and back-propagation (BP) required to compute $C^{(t)}$, $	\frac{\partial \mathcal{L}}{\partial U^{(t)}}$, and $	\frac{\partial \mathcal{L}}{\partial h^{(t-1)}}$. See Appendix \ref{B} for more detail about how this algorithm is derived using Theorem \ref{th2}. 
	
	\begin{algorithm}[tb]
		\caption{Local forward and backward propagations at time step $t$. For a matrix $A$, $A_{*,k}$ denotes the $k$-th column.}
		\label{alg}
		\begin{algorithmic}[1]
			\STATE {\bfseries Inputs:} $h^{(t-1)}$, $\frac{\partial \mathcal{L}}{\partial C}$, $U=(\mathbf{u}_n|\dots| \mathbf{u}_{n-m+1})$.
			\STATE {\bfseries Outputs:}  $\frac{\partial \mathcal{L}}{\partial U^{(t)}}$, $\frac{\partial \mathcal{L}}{\partial h^{(t-1)}}$, $C^{(t)}=Wh^{(t-1)}$\\
			\STATE {\bfseries Require:} $G \in \mathbb{R}^{n\times m}$, $g \in \mathbb{R}^{n}$, $H \in \mathbb{R}^{n \times (m+1)}$ 
			\STATE $N \leftarrow (\norm{\mathbf{u}_n}^2, \dots, \norm{\mathbf{u}_{n-m+1}}^2)$
			\STATE $H_{*,m+1}  \leftarrow h^{(t-1)}$
			\STATE $g \leftarrow \frac{\partial \mathcal{L}}{\partial C}$ 
			\FOR[Local Forward Propagation]{$k=m$ {\bfseries to} $1$} 
			\STATE $\tilde{h}_{k} \leftarrow \frac{2}{N_{k} } U_{*, k}' H_{*,k+1}$
			\STATE $H_{*,k} \leftarrow H_{*,k+1} - \tilde{h}_{k}  U_{*,k}$
			\ENDFOR
				\FOR[Local Backward Propagation]{$k=1$ {\bfseries to} $m$}
			\STATE \label{aa} $\tilde{C}_{k} \leftarrow \frac{2}{N_{k} } U_{*, k}' g$
			\STATE $g\leftarrow g- \tilde{C}_{k}  U_{*,k}$
			\STATE $G_{*, k} \leftarrow - \tilde{h}_k g
		 - \tilde{C}_k H_{*, k+1}$
			\ENDFOR
			\STATE $C^{(t)} \leftarrow H_{*, 1}$
			\STATE $\frac{\partial\mathcal{L}}{\partial h^{(t-1)}}\leftarrow  g$
			\STATE $\frac{\partial \mathcal{L}}{\partial U^{(t)}} \leftarrow G$
		\end{algorithmic}
	\end{algorithm}

 In the next section we analyse the time and space complexities of this algorithm.
	\subsection{Time and Space complexity}
		At each time step $t$, the flop count required by Algorithm \ref{alg} is $(13n + 2)m$;
		 $6nm$ for the one-step FP and $(7n+2)m$ for the one-step BP. Note that the vector $N$ only needs to be calculated at one time step. This reduces the flop count at the remaining time steps to $(11n+3)m$. The fact that the matrix $U$ has all zeros in its upper triangular part can be used to further reduce the total flop count to $(11n-3m+5)m$; $(4n-m+2)m$ for the one-step FP and $(7n-2m+3)m$ for the one-step BP. See Appendix \ref{C} for more details.
		
		Note that if the values of the matrices $H$, defined in Algorithm \ref{alg}, are first stored during a ``global'' FP (i.e. through all time steps), then used in the BP steps, the time complexity\footnote{We considered only the time complexity due to computations through the hidden-to-hidden connections of the network.} for a global FP and BP using one input sequence of length $T$ are, respectively, $\approx 3n^2T$ and $\approx 5n^2T$, when $m \approx n$ and $n \gg 1$. In contrast with the sRNN case with $n$ hidden units, the global FP and BP have time complexities $\approx 2n^2T$ and $\approx 3n^2T$. Hence, when $m \approx n$, the FP and BP steps using our parametrisation require only about twice more flops than the sRNN case with the same number of hidden units. 
		
		Note, however, that storing the values of the matrices $H$ at all time steps requires the storage of $mnT$ values for one sequence of length $T$, compared with $nT$ when only the hidden states $\{h^{(t)}\}_{t=1}$ are stored. When $m \gg 1$ this may not be practical. One solution to this problem is to generate the matrices $H$ locally at each BP step using $U$ and $h^{(t-1)}$. This results in a global BP complexity of $(11n-3m+5)mT$. Table \ref{t3}  summarises the flop counts for the FP and BP steps. Note that these flop counts are for the case when $m \leq n-1$. When $m=n$, the complexity added due to the multiplication by $\mathcal{H}_1(u_1)$ is negligible. 

		\begin{table}[]
		\centering
		\begin{tabular}{c|cc}
			\hline
			 &     \multicolumn{2}{c}{Flop counts}  \\
			Model   & FP & BP  \\
				\hline
			oRNN $(n, m)$    & $(4n-m+2)m$  &   $(7n-2m+3)m$ \\
			sRNN  $(n)$  & $2n^2 - n$  & $3n^2 - n$  \\
			\hline 
		\end{tabular}
		\caption{Summary of the flop counts due to the computations of one-step FP and BP through the hidden-to-hidden connections. The BP flop counts for the oRNN case assumes that the $H$ matrices (see Algorithm \ref{alg}) are not locally generated during the BP steps. Otherwise, the flop count would be $(11n-3m+5)m$.}
		\label{t3}
	\end{table}
	
	\subsection{Extension to the Unitary case}
		Although we decided to focus on the set of real-valued orthogonal matrices, for the reasons given in Section \ref{se2}, our parametrisation can readily be modified to apply to the general unitary case. 
		
		Let $\hat{\mathcal{H}}_k: \mathbb{C}^{k} \rightarrow \mathbb{C}^{n\times n}$, $2 \leq k\leq n$, be defined by Equation \eqref{eq:19} where the transpose sign $'$  is replaced by the conjugate transpose ${}^*$. Furthermore, let $\hat{\mathcal{H}}_1: \mathbb{R}^n \rightarrow \mathbb{C}^{n \times n}$ be defined as $\hat{\mathcal{H}}_1(\theta) =\mbox{diag}(e^{\mathbf{i} \theta_1}, \dots, e^{\mathbf{i} \theta_n})$. With the new mappings $\{\hat{\mathcal{H}}_k\}_{k=1}^{k=n}$, we have the following corollary.
		\begin{corollary}
			Let $\hat{\mathcal{M}}_1$ be the mapping defined as 
			\begin{align*}
			\hat{\mathcal{M}}_1: \mathbb{R}^n \times \mathbb{C}^2 \times \dots \times \mathbb{C}^n  &\rightarrow \mathbb{C}^{n\times n}\\
			(\theta, \mathbf{u}_2, \dots, \mathbf{u}_n) \mapsto& \hat{\mathcal{H}}_n(\mathbf{u}_n) \dots \hat{\mathcal{H}}_2(\mathbf{u}_2) \hat{\mathcal{H}}_{1}(\theta).
			\end{align*}
			
		The image of $\hat{\mathcal{M}}_1$ spans the full set of unitary matrices $\mathbf{U}(n)$ and any point on its image is a unitary matrix.
		\end{corollary}

	\section{\textbf{Experiments}}
	All RNN models were implemented using the python library \textit{theano} \cite{2016arXiv160502688short}. For efficiency, we implemented the one-step FP and BP algorithms described in Algorithm \ref{alg} using C code\footnote{Our implementation can be found at \url{https://github.com/zmhammedi/Orthogonal_RNN}.}. We tested the new parametrisation on five different datasets all having long-term dependencies. 
	We call our parametrised network oRNN (for orthogonal RNN). We set its activation function to the \lstinline{leaky_ReLU} defined as $\phi(x) = \max(\frac{x}{10}, x)$. To ensure that the transition matrix of the oRNN is always orthogonal, we set the scalar $u_1$ to -1 if $u_1 \leq 0$ and 1 otherwise after each gradient update. Note that the parameter matrix $U$ in Equation \eqref{eq101} has all zeros in its upper triangular part. Therefore, after calculating the gradient of a loss with respect to $U$ (i.e. $\frac{\partial \mathcal{L}}{\partial U}$), the values in the upper triangular part are set to zero. 
	
	For all experiments, we used the \textit{adam} method for stochastic gradient descent \cite{DBLP:journals/corr/KingmaB14}.
	We initialised all the parameters using uniform distributions similar to \cite{arjovsky2015unitary}.  The biases of all models were set to zero, except for the forget bias of the LSTM, which we set to 5 to facilitate the learning of long-term dependencies \cite{DBLP:journals/corr/KoutnikGGS14}. 	
	
		\subsection{Sequence generation}	
	In this experiment, we followed a similar setting to \cite{DBLP:journals/corr/KoutnikGGS14} where we trained RNNs to encode song excerpts. We used the track \textit{Manyrista} from album \textit{Musica Deposita} by \textit{Cuprum}. We extracted five consecutive excerpts around the beginning of the song, each having $800$ data points and corresponding to 18ms with a 44.1Hz sampling frequency. We trained an sRNN, LSTM, and oRNN for 5000 epochs on each of the pieces with five random seeds. For each run, the lowest Normalised Mean Squared Error (NMSE) during the 5000 epochs was recorded. For each model, we tested three different hidden sizes. The total number of parameters $N_p$ corresponding to these hidden sizes was approximately equal to $250$, $500$, and $1000$. For the oRNN, we set the number of reflection vectors to the hidden size for each case, so that the transition matrix is allowed to span the full set of orthogonal matrices. The results are shown in Figures \ref{f2} and \ref{f1}. All the learning rates were set to $10^{-3}$. The orthogonal parametrisation outperformed the sRNN and performed on average better than the LSTM. 
	
	\begin{figure}[]
		\centering
		\includegraphics[trim=.5cm 0.5cm .5cm 0.8cm, clip,  width=.9\linewidth]{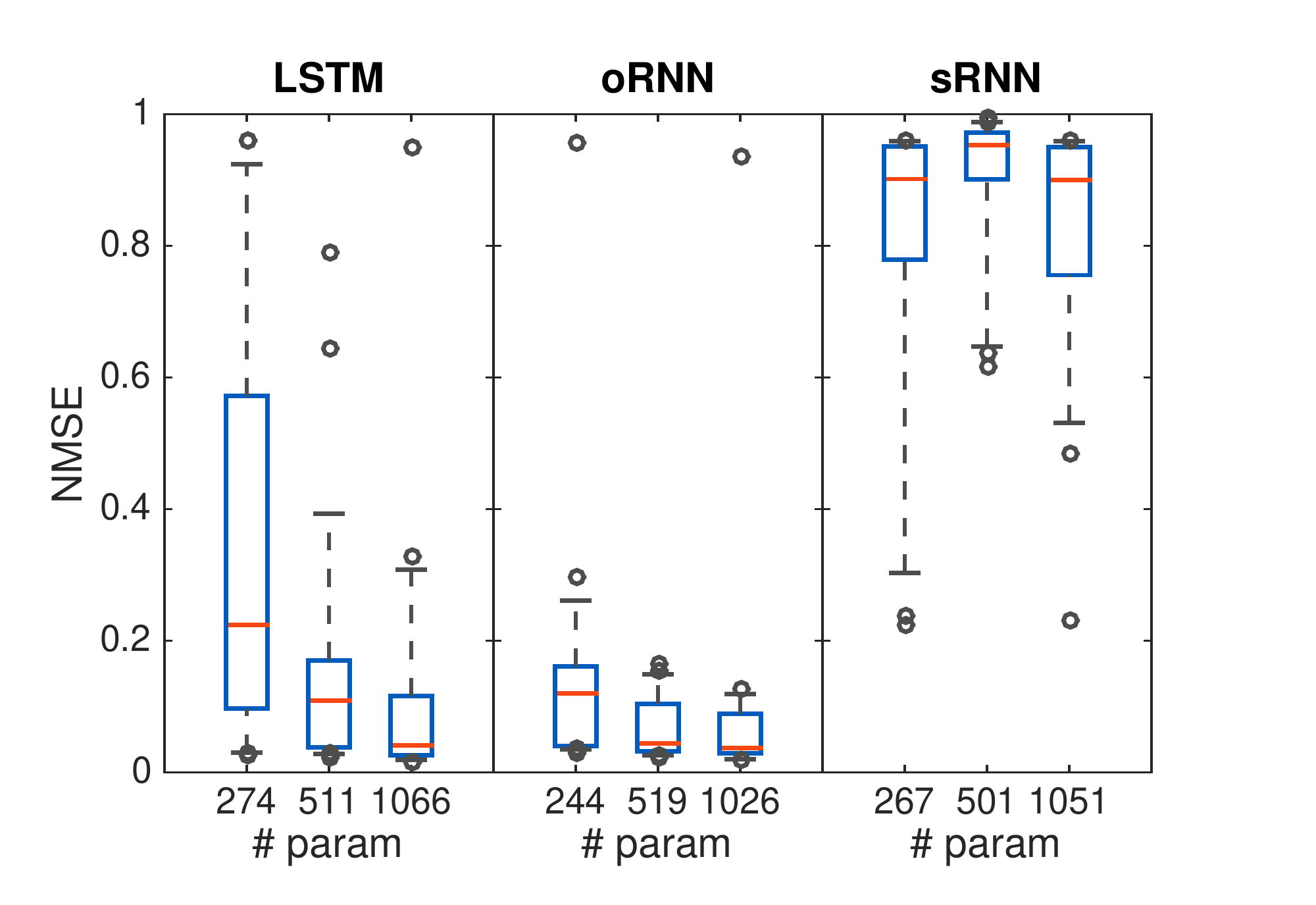}
		\caption{Sequence generation task: The plots show the NMSE distributions for the different models with respect to the total number of parameters for the sequence generation task. The horizontal red lines represent the medians of the NMSE over 25 data points (i.e. five seeds for each of the five song excerpts). The solid rectangles and the dashed bars represent the $[25\%-75\%]$ and  $[9\%-91\%]$ confidence intervals respectively.} 
		\label{f2}
	\end{figure}
	\begin{figure}[]
		\centering
		\includegraphics[trim=0cm 0.5cm 0cm 1cm, clip,  width=0.9\linewidth]{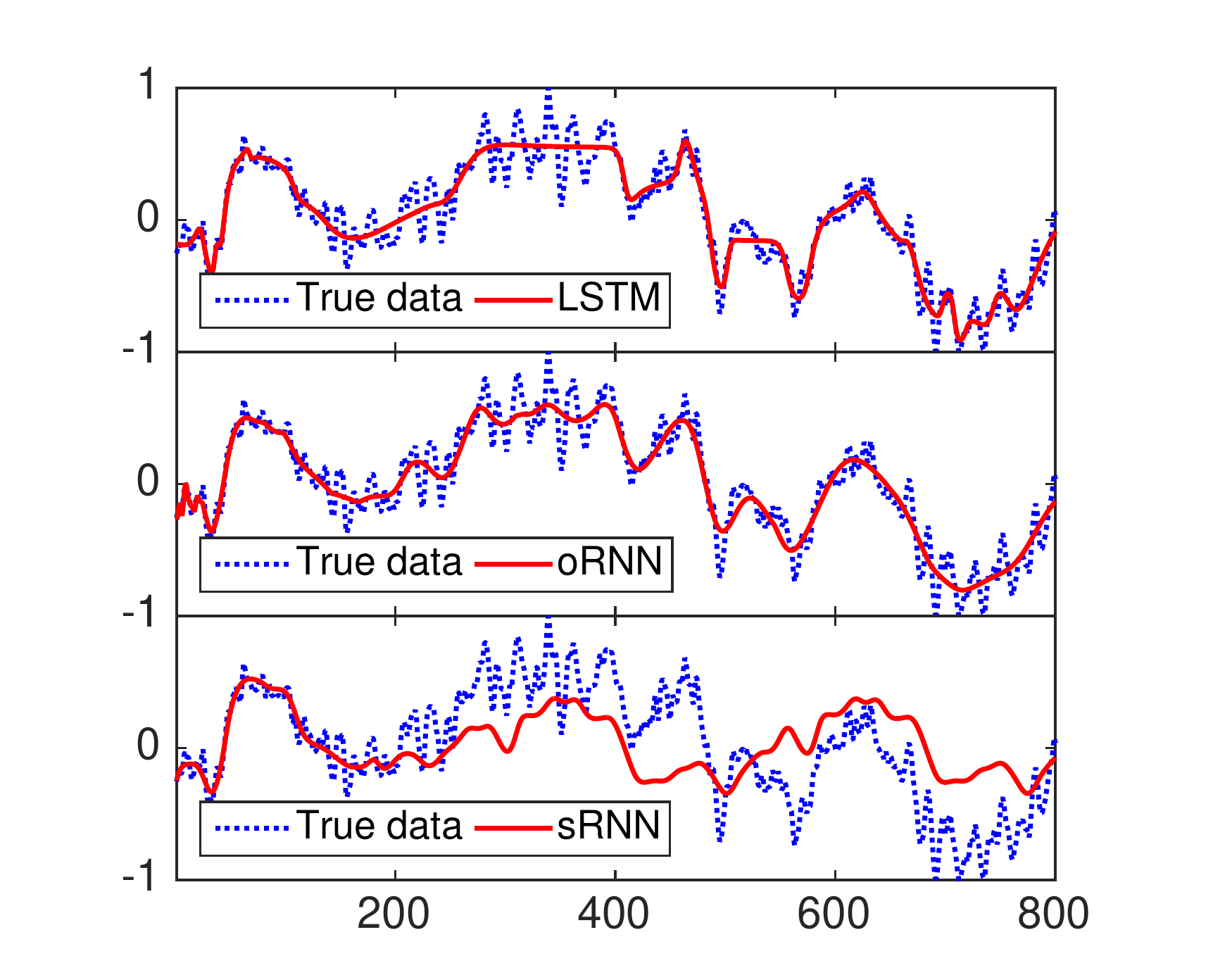}
		\caption{Sequence generation task: The RNN-generated sequences against the true data for one of the five excerpts used. We only displayed the best performing models for $N_{p} \simeq 1000$.}
		\label{f1}
	\end{figure}

		\subsection{Addition Task}
	In this experiment, we followed a similar setting to \cite{arjovsky2015unitary}, where the goal of the RNN is to output the sum of two elements in the first dimension of a two-dimensional sequence. The location of the two elements to be summed are specified by the entries in the second dimension of the input sequence. In particular, the first dimension of every input sequence consists of random numbers between 0 and 1. The second dimension has all zeros except for two elements equal to 1. The first unit entry is located in the first half of the sequence, and the second one in the second half. We tested two different sequence lengths $T=400, 800$. All models were trained to minimise the Mean Squared Error (MSE). The baseline MSE for this task is 0.167; for a model that always outputs one. 
	
	\begin{figure}[h]
		\centering
		\includegraphics[trim=0cm 7cm 0cm 6.8cm, clip,  width=0.9\linewidth]{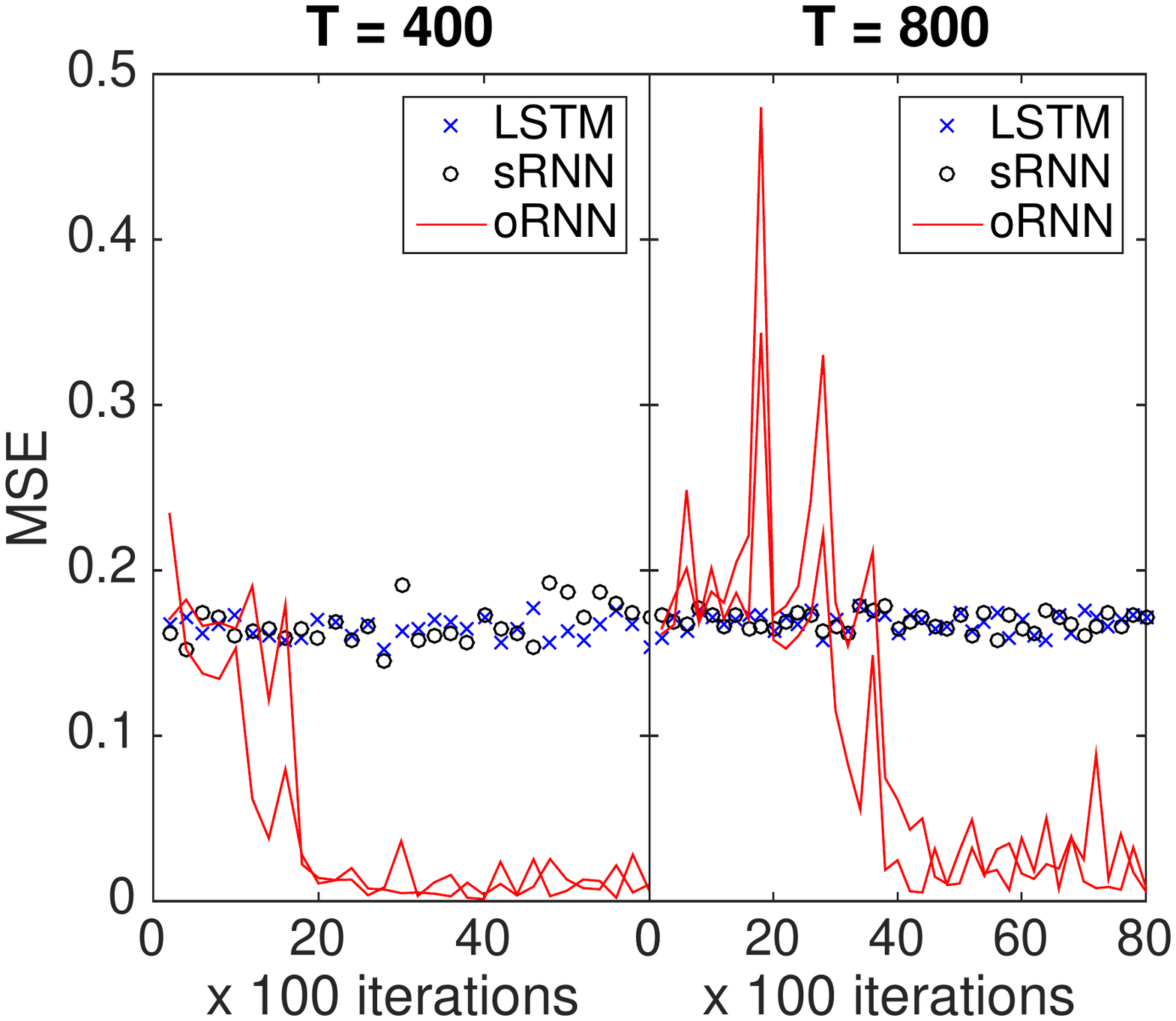}
		\caption{Addition task: For each lag $T$, the red curves represent two runs of the oRNN model with two different random initialisation seeds. LSTM and sRNN did not beat the baseline MSE.}
		\label{fig_sdim}
	\end{figure}
	
	We trained an oRNN with $n=128$ hidden units and $m=16$ reflections. We trained an LSTM and sRNN with hidden sizes 28 and 54, respectively, corresponding to a total number of parameters $\approx 3600$ (i.e. same as the oRNN model). We chose a batch size of 50, and after each iteration, a new set of sequences was generated randomly. The learning rate for the oRNN was set to 0.01. Figure \ref{fig_sdim} displays the results for both lags. 
	
	The oRNN was able to beat the baseline MSE in less than 5000 iterations for both lags and for two different random initialisation seeds. This is in line with the results of the unitary RNN \cite{arjovsky2015unitary}.

    \begin{figure}[]
	\centering
	\includegraphics[trim=.5cm 7.5cm .5cm 7.5cm, clip,  width=0.9\linewidth]{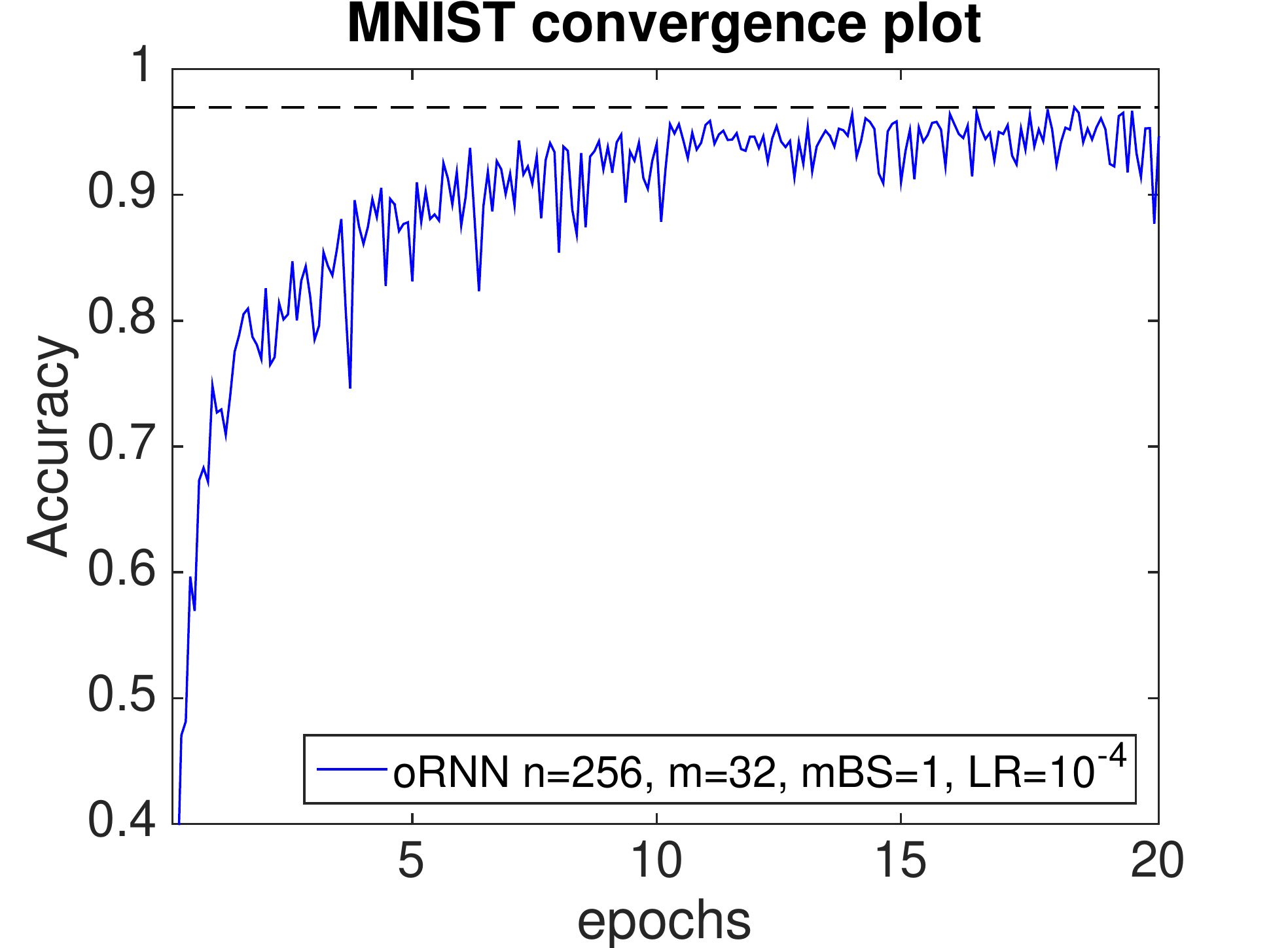}
	\caption{MNIST experiment: Validation accuracy of the oRNN in Table \ref{t1} as a function of the number of epochs. mBS and LR in the legend stand for mini-batch size and learning rate respectively.} 
	\label{f8}
\end{figure}

	\subsection{Pixel MNIST}
    In this experiment, we used the MNIST image dataset. We split the dataset into training (55000 instances), validation (5000 instances), and test sets (10000 instances). We trained oRNNs with $n\in\{128,256\}$ and $m \in \{16, 32,64\}$, where $n$ and $m$ are the number of hidden units and reflections vectors respectively, to minimise the cross-entropy error function. We experimented with (mini-batch size, learning rate) $\in \{(1, 10^{-4}), (50, 10^{-3})\}$.
    
    	\begin{table*}[]
    	\centering
    	\begin{tabular}{ccccc}
    		\hline
    		Model &  hidden size & Number of  & validation  & test \\
    		&(\# reflections) & parameters & accuracy& accuracy  \\
    		\hline
    		oRNN  & 256 (m=32)  & $\approx$11K &  97.0 \%  & 97.2 \% \\
    		RNN \cite{vorontsov2017orthogonality}  &128   & $\approx$35K & -  & 94.1 \% \\
    		uRNN \cite{arjovsky2015unitary} &  512  & $\approx$16K  & - & 95.1 \% \\
    		RC uRNN \cite{wisdom2016full} &  512  & $\approx$16K  & - & 97.5 \% \\
    		FC uRNN \cite{wisdom2016full} &  116  & $\approx$16K  & -  & 92.8 \% \\
    		\hline 
    	\end{tabular}
    	\caption{Results summary for the MNIST digit classification experiment and comparison with the uRNN results available in the literature. 'FC' and 'RC' stand for Full-Capacity and Restricted Capacity respectively. The oRNN with $n=256$ was trained using a mini-batch size of 1 and a learning rate of $10^{-4}$.}
    	\label{t1}
    \end{table*}

Table \ref{t1} compares the test performance of our best model against results available in the literature for unitary/orthogonal RNNs. Despite having fewer total number of parameters, our model performed better than three out the four models selected for comparison (all having $\geq 16K$ parameters). Figure \ref{f8} shows the validation accuracy as a function of the number of epochs of our oRNN model in Table \ref{t1}. Figure \ref{f4} shows the effect of varying the number of reflection vectors $m$ on the performance. 
    \begin{figure}[]
	\centering
	\includegraphics[trim=.5cm 7.5cm .5cm 7.5cm, clip,  width=0.9\linewidth]{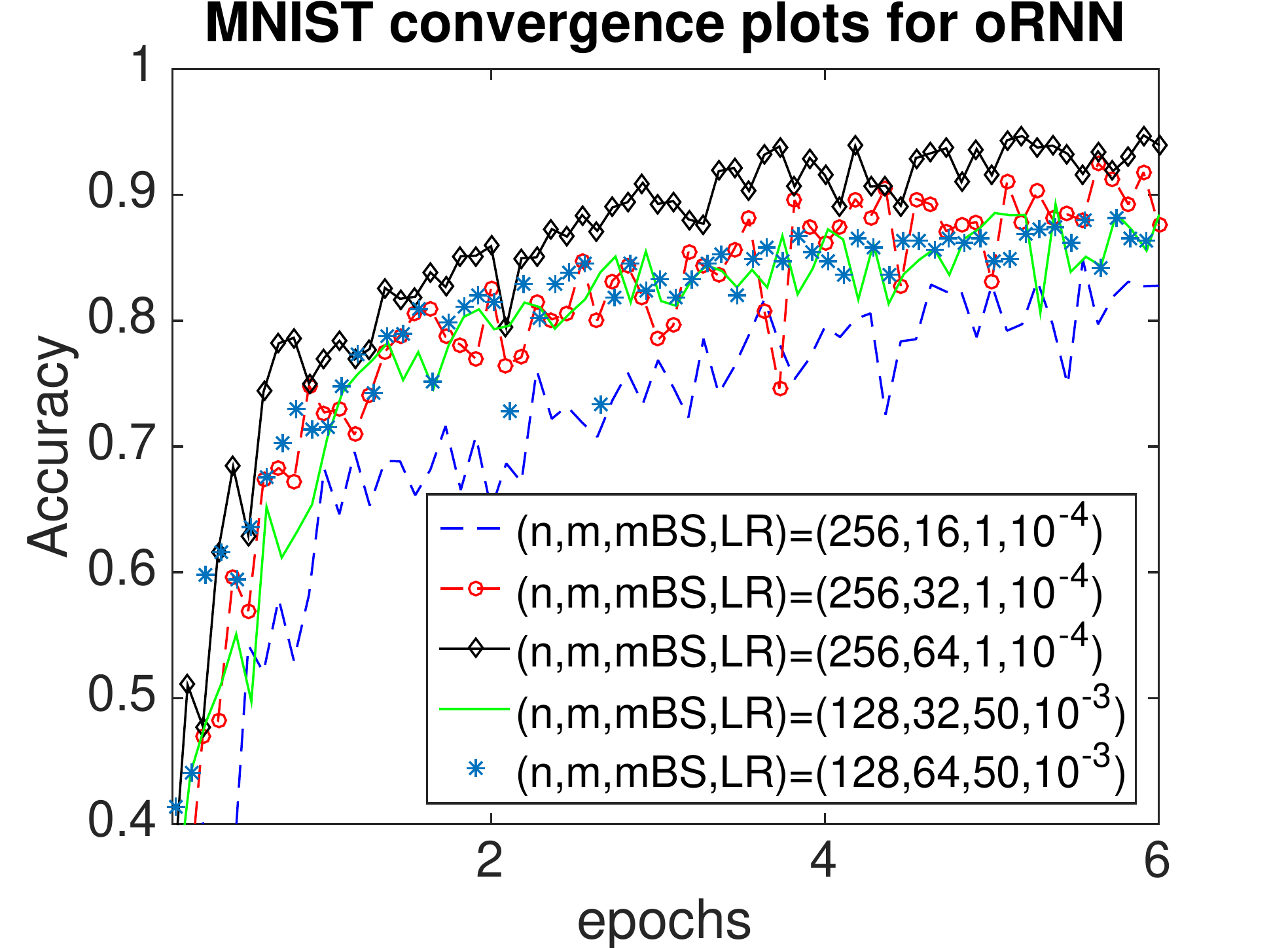}
	\caption{MNIST experiment: Effect of varying the number of reflection vectors $m$ on the validation accuracy and speed of convergence. mBS (resp, LR) stands for mini-batch size (resp, learning rate).} 
	\label{f4}
\end{figure}

	\subsection{Penn Tree Bank}
    In this experiment, we tested the oRNN on the task of character level prediction using the Penn Tree Bank Corpus. The data was split into training (5017K characters), validation (393K characters), and test sets (442K characters). The total number of unique characters in the corpus was 49. The vocabulary size was 10K and any other words were replaced by the special token \lstinline{<unk>}. The number of characters per instance (i.e. char/line) in the training data ranged between 2 and 518 with an average of 118 char/line. We trained an oRNN and LSTM with hidden units 512 and 183 respectively, corresponding to a total of $\approx 180$K parameters, for 20 epochs. We set the number of reflections to 510 for the oRNN. The learning rate was set to 0.0001 for both models with a mini-batch size of 1.

Similarly to \cite{pascanu2013difficulty} we considered two tasks: one where the model predicts one character ahead and the other where it predicts a character five steps ahead. It was suggested that solving the later task would require the learning of longer term correlations in the data rather than the shorter ones. 
Table \ref{t5} summarises the test results. The oRNN and LSTM  performed similarly to each other on the one-step head prediction task. Whereas on the five-step ahead prediction task, the LSTM was better. The performance of both models on this task was close to the state of the art result for RNNs $3.74$ bpc  \cite{pascanu2013difficulty}.

Nevertheless, our oRNN still outperformed the results of \cite{vorontsov2017orthogonality} which used both soft and hard orthogonality constraints on the transition matrix. Their RNN was trained on 99\% of the data (sentences with $\leq$ 300 characters) and had the same number of hidden units as the oRNN in our experiment. The lowest test cost achieved was 2.20(bpc) for the one-step-ahead prediction task.     
	\begin{table}
		\centering
		\begin{tabular}{ccccc}
			\hline
			& Model ($n_h$)	& $\#$ params  & valid cost & test cost \\
			\hline
			\parbox[t]{2mm}{\multirow{2}{*}{\rotatebox[origin=c]{90}{1-step}}} & oRNN (512) & 183K & 1.73 & 1.68    \\ 
			& LSTM (184)  & 181K & 1.69 & 1.68 \\
			\hline
			\parbox[t]{2mm}{\multirow{2}{*}{\rotatebox[origin=c]{90}{5-step}}}&	oRNN (512) & 183K & 3.85 & 3.85 \\ 
			&  LSTM (183) & 180K & 3.81 & 3.8  \\
			\hline
		\end{tabular}
		\caption{Results summary for the Penn Tree bank experiment.}
			\label{t5}
	\end{table}

	\subsection{Copying task}
    We tested our model on the copy task described in details in \cite{Gers2001, arjovsky2015unitary}. Using an oRNN with the \lstinline{leaky_ReLU} we were not able to reproduce the same performance as the uRNN \cite{arjovsky2015unitary, wisdom2016full}. However, we were able to achieve a comparable performance when using the \lstinline{OPLU} activation function \cite{chernodub2016norm}, which is a norm-preserving activation function. In order to explore whether the poor performance of the oRNN was only due to the activation function, we tested the same activation as the uRNN (i.e. the real representation of \lstinline{modReLU} defined in Equation \eqref{eq104}) on the oRNN. This did not improve the performance compared to the \lstinline{leaky_ReLU} case suggesting that the block structure of the uRNN transition matrix, when expressed in the real space (see Section \ref{se2}), may confer special benefits in some cases. 
	
	\section{Discussion}
	\label{se4}
	In this work, we presented a new parametrisation of the transition matrix of a recurrent neural network using Householder reflections. This method allows an easy and computationally efficient way to enforce an orthogonal constraint on the transition matrix which then ensures that exploding gradients do not occur during training. Our method could also be applied to other deep neural architectures to enforce orthogonality between hidden layers. Note that a ``soft'' orthogonal constraint could also be applied using our parametrisation by, for example, allowing $u_1$ to vary continuously between -1 and 1.  
	
	It is important to note that our method is particularly advantageous for stochastic gradient descent when the mini-batch size is close to 1. In fact, if $B$ is the mini-batch size and $T$ is the average length of the input sequences, then a network with $n$ hidden units trained using other methods \cite{vorontsov2017orthogonality,wisdom2016full, hyland2016learning} that enforce orthogonality (see Section \ref{se1}), would have time complexity $\mathcal{O}(BTn^2 + n^3)$. Clearly when $BT \gg n$ this becomes $\mathcal{O}(BTn^2)$, which is the same time complexity as that of the sRNN and oRNN (with $m=n$). In contrast with the case of fully connected deep forward networks with no weight sharing between layers (\# layer $=L$), the time complexity using our method is $\mathcal{O}(BLnm)$ whereas other methods discussed in this work (see Section \ref{se1}) would have time complexity $\mathcal{O}(BLn^2 + L n^3)$. The latter methods are less efficient in this case since $B \gg n$ is less likely to be the case compared with $BT \gg n$ when using SGD.
	
	From a performance point of view, further experiments should be performed to better understand the difference between the unitary versus orthogonal constraint.

	\section*{Acknowledgment}
	The authors would like to acknowledge Department of State Growth Tasmania for partially funding this work through SenseT. 
We would also like to thank Christfried Webers for his valuable feedback.
	\bibliographystyle{icml2017}
	\bibliography{bibliography}
	
	\clearpage
	
\begin{appendix}
	\section{Proofs}
	\label{A}
	\begin{proof}[Sketch of the proof for Theorem 1]			
		We need to the show that for every $\tilde{Q} \in \mathbf{O}(n)$, there exits a tuple of vectors $(\mathbf{u}_1, \dots, \mathbf{u}_n) \in \mathbb{R}\times \dots \times \mathbb{R}^n$ such that $\tilde{Q} = \mathcal{M}_1(\mathbf{u}_1, \dots, \mathbf{u}_n)$.
		Algorithm \ref{alg1} shows how a QR decomposition can be performed using the matrices $\{\mathcal{H}_k(\mathbf{u}_k)\}_{k=1}^{n}$ while ensuring that the upper triangular matrix $R$ has positive diagonal elements. If we apply this algorithm to an orthogonal matrix $\tilde{Q}$, we get a tuple $(\mathbf{u}_1, \dots, \mathbf{u}_n)$ which satisfies
		\begin{align*}
		Q R  = \mathcal{H}_n(\mathbf{u}_n) \dots \mathcal{H}_1(\mathbf{u}_1) R = \tilde{Q} .
		\end{align*}
		
		Note that the matrix $R$ must be orthogonal since $R = Q' \tilde{Q}$. Therefore, $R = I$, since the only upper triangular matrix with positive diagonal elements is the identity matrix. Hence, we have
		\begin{align*}
		\mathcal{M}_1(\mathbf{u}_1, \dots, \mathbf{u}_n) = \mathcal{H}_n(\mathbf{u}_n) \dots \mathcal{H}_1(\mathbf{u}_1) = \tilde{Q}.
		\end{align*}
		
	\end{proof}
	
	\begin{algorithm}[h]
		\caption{QR decomposition using the mappings $\{\mathcal{H}_k\}$. For a matrix $B \in \mathbb{R}^{n \times n}$, $\{B_{k,k}\}_{ 1\leq k\leq n}$ denote its diagonal elements, and $B_{k..n,k} = (B_{k,k},\dots,B_{n,k})'\in \mathbb{R}^{n-k+1}$.}
		\label{alg1}
		\begin{algorithmic}
			\REQUIRE $A \in \mathbb{R}^{n \times n} $ is a full-rank matrix.
			\ENSURE $Q$ and $R$ where $Q = \mathcal{H}_n(\mathbf{u}_n)\dots \mathcal{H}_1(\mathbf{u}_1)$ and $R$ is upper triangular with positive diagonal elements such that $A = QR$  
			\STATE $R \leftarrow A$ 
			\STATE $Q \leftarrow I$  \COMMENT{Initialise Q to the identity matrix}
			\FOR{$k=1$ to $n-1$}   
			\IF{$R_{k,k} == \norm{R_{k..n,k}} $}  
			\STATE $\mathbf{u}_{n-k+1} = (0, \dots, 0, 1)' \in \mathbb{R}^{n-k+1}$  
			\ELSE
			\STATE $\mathbf{u}_{n-k+1} \leftarrow R_{k..n,k} - \norm{R_{k..n,k}}(1,0, \dots, 0)'$ 
			\STATE $ \mathbf{u}_{n-k+1} \leftarrow \mathbf{u}_{n-k+1} / \norm{\mathbf{u}_{n-k+1}}$
			\ENDIF
			\STATE $R \leftarrow \mathcal{H}_{n-k+1}(\mathbf{u}_{n-k+1}) R$
			\STATE $Q \leftarrow Q \mathcal{H}_{n-k+1}(\mathbf{u}_{n-k+1}) $ 
			\ENDFOR
			\STATE $\mathbf{u}_1 = sgn(R_{n,n}) \in \mathbb{R}$
			\STATE $R \leftarrow \mathcal{H}_1(\mathbf{u}_1) R$ 
			\STATE $Q \leftarrow Q \mathcal{H}_1(\mathbf{u}_1)$ 
		\end{algorithmic}
	\end{algorithm}

	\begin{lemma}{\cite{giles2008extended}}
		Let $A$, $B$, and $C$ be real or complex matrices, such that $C = f(A,B)$ where $f$ is some differentiable mapping. 
		Let $\mathcal{L}$ be some scalar quantity which depends on $C$. Then we have the following identity
		\begin{align*}
		\mbox{Tr}(\overline{C}' dC) =     \mbox{Tr}(\overline{A}' dA) +     \mbox{Tr}(\overline{B}' dB),
		\end{align*}
		where $dA$, $dB$, and $dC$ represent infinitesimal perturbations and \[
		\overline{C} \coloneqq \frac{\partial \mathcal{L}}{\partial C}, \; \overline{A} \coloneqq \left[\frac{\partial C}{\partial A}\right]' \frac{\partial \mathcal{L}}{\partial C} ,\;\overline{B} \coloneqq \left[\frac{\partial C}{\partial B}\right]' \frac{\partial \mathcal{L}}{\partial C}.\]
	\end{lemma}

	\begin{proof}[Proof of Theorem 2]
		Let $C = h - U T^{-1}U'h$ where $(U, h) \in \mathbb{R}^{n \times m}\times \mathbb{R}^n$ and $T = \mbox{striu}(U' U) + \frac{1}{2} \mbox{diag}(U'U)$. Notice that the matrix $T$ can be written using the Hadamard product  as follows
		\begin{align}
		\label{eq44}
		T =  B \circ (U'U),
		\end{align}
		where $B= \mbox{striu}(J_m) + \frac{1}{2} I_m$ and $J_m$ is the $m\times m$ matrix of all ones.
		
		Calculating the infinitesimal perturbations of $C$ gives 
		\begin{align*}
		dC = & (I - U T^{-1} U')dh \\ & - dU T^{-1} U' h  - U T^{-1} dU'  h \\ &  + U T^{-1} dT T^{-1} U' h.
		\end{align*}
		Using Equation \eqref{eq44} we can write 
		\begin{align*}
		dT = B \circ (dU'U + U' dU).
		\end{align*}
		
		By substituting this back into the expression of $dC$, multiplying the left and right-hand sides by $\overline{C}'$, and applying the trace we get
		\begin{align*}
		\mbox{Tr}(\overline{C}'&dC) =  \mbox{Tr}(\overline{C}'(I - U T^{-1} U')dh) \\ & - \mbox{Tr}(\overline{C}'dU T^{-1} U' h)  - \mbox{Tr}(\overline{C}'U T^{-1} dU'  h) \\ &  + \mbox{Tr}(\overline{C}'U T^{-1} (B \circ (dU'U + U' dU)) T^{-1} U' h).
		\end{align*}
		
		Now using the identity $\mbox{Tr}(AB) = \mbox{Tr}(BA)$, where the second dimension of A agrees with the first dimension of B, we can rearrange the expression of $\mbox{Tr}(\overline{C}'dC)$ as follows 
		\begin{align*}
		\mbox{Tr}(\overline{C}'&dC) =  \mbox{Tr}(\overline{C}'(I - U T^{-1} U')dh) \\ & - \mbox{Tr}(T^{-1} U' h\overline{C}'dU )  - \mbox{Tr}(h\overline{C}'U T^{-1} dU'  ) \\ &  + \mbox{Tr}(T^{-1} U' h\overline{C}'U T^{-1} (B \circ (dU'U + U' dU)) ).
		\end{align*}
		
		To simplify the expression, we will use the short notations 
		\begin{align*}
		\tilde{C} &= (T')^{-1}U' \overline{C}, \\
		\tilde{h} &= T^{-1}U' h,
		\end{align*}
		$\mbox{Tr}(\overline{C}'dC)$ becomes 
		\begin{align*}
		\mbox{Tr}(\overline{C}'dC) &=  \mbox{Tr}((\overline{C}' -\tilde{C}'U')dh) \\ & - \mbox{Tr}(\tilde{h}\overline{C}'dU )  - \mbox{Tr}(h\tilde{C}' dU'  ) \\ &  + \mbox{Tr}(\tilde{h}\tilde{C}' (B \circ (dU'U + U' dU)) ).
		\end{align*}
		
		Now using the two following identities of the trace
		\begin{align*}
		\mbox{Tr}(A') &= \mbox{Tr}(A), \\
		\mbox{Tr}(A  (B \circ C)) &= \mbox{Tr}((A \circ B')  C)),
		\end{align*}
		we can rewrite $\mbox{Tr}(\overline{C}'dC)$ as follows 
		\begin{align*}
		\mbox{Tr}(\overline{C}'dC) = & \mbox{Tr}((\overline{C}' -\tilde{C}'U')dh) \\ & - \mbox{Tr}(\tilde{h}\overline{C}'dU )  - \mbox{Tr}(h\tilde{C}' dU'  ) \\ &  + \mbox{Tr}((\tilde{h}\tilde{C}' \circ B') dU'U) \\ & + \mbox{Tr}((\tilde{h}\tilde{C}' \circ B')U' dU ).
		\end{align*}
		
		By rearranging and taking the transpose of the third and fourth term of the right-hand side we obtain 
		\begin{align*}
		\mbox{Tr}(\overline{C}'dC) = & \mbox{Tr}((\overline{C}' -\tilde{C}'U')dh) \\ & - \mbox{Tr}(\tilde{h}\overline{C}'dU )  - \mbox{Tr}(\tilde{C} h' dU  ) \\ &  + \mbox{Tr}(((\tilde{C} \tilde{h}')\circ B) U' dU) \\ & + \mbox{Tr}(((\tilde{h}\tilde{C}') \circ B')U' dU ).
		\end{align*}
		
		Factorising by $dU$ inside the $\mbox{Tr}$ we get 
		\begin{align*}
		\mbox{Tr}&(\overline{C}'dC) =  \mbox{Tr}((\overline{C}' -\tilde{C}'U')dh) - \\ & \mbox{Tr}( (\tilde{h}\overline{C}' + \tilde{C} h' - \left[(\tilde{C} \tilde{h}')\circ B+  (\tilde{h}\tilde{C}') \circ B' \right] U')  dU ).
		\end{align*}
		
		Using lemma 1 we conclude that 
		\begin{align*}
		\overline{U}  =& U \left[ (\tilde{h} \tilde{C}') \circ  B' + (\tilde{C} \tilde{h}') \circ  B \right]  - \overline{C}  \tilde{h}' - h \tilde{C}',   \\
		\overline{h} = &\overline{C} - U \tilde{C}.
		\end{align*}
	\end{proof}
	
	\begin{proof}[Sketch of the proof for Corollary 1]			
		For any nonzero complex valued vector $x \in \mathbb{C}^n$, if we chose $u = x + e^{i\theta}\norm{x} e_1 $ and $H = -e^{-i\theta}(I - 2  \frac{u u^*}{\norm{u}^2})$, where $\theta \in \mathbb{R}$ is such that $x_1=e^{i\theta}|x_1|$, we have \cite{mezzadri2006generate}
		\begin{align}
		Hx = ||x|| e_1
		\end{align}
		Taking this fact into account, a similar argument to that used in the proof of Theorem 1 can be used here. 
	\end{proof}

	\section{Algorithm Explanation}
	\label{B}
	Let $U \coloneqq (v_{i,j})_{\substack{1\leq i \leq n}}^{1 \leq j \leq m}$.  Then the element of the matrix $T\coloneqq \mbox{striu}(U' U) + \frac{1}{2} \mbox{diag} (U'U) $ can be expressed as
	\begin{align*}
	t_{i,j} = \llbracket  i\leq j \rrbracket \frac{\sum_{k=j}^n v_{k,i} v_{k,j}}{1+\delta_{i,j} },
	\end{align*}
	where $\delta_{i,j}$ is the Kronecker delta and $\llbracket \cdot \rrbracket$ is the Iversion bracket (i.e. $\llbracket     p \rrbracket $ = 1 if p is true and $\llbracket     p \rrbracket$ = 0 otherwise).
	
	In order to compute the gradients in Equations \eqref{eq5} and \eqref{eq6}. we first need to compute $\tilde{h} = T^{-1} U' h$ and $\tilde{C} = (T')^{-1}U' \frac{\partial \mathcal{L}}{ \partial C}$. This is equivalent to solving the triangular systems of equations $T \tilde{h}=U' h$ and $T' \tilde{C}=U' \frac{\partial \mathcal{L}}{\partial C}$.
	
	\textbf{Solving the triangular system $T \tilde{h}=U' h$}. For $1 \leq  k \leq m$, we can express the $k$-th row of this system as 
	\begin{align}
	t_{k,k} \tilde{h}_k&+  \sum_{j=k+1}^m t_{k,j} \tilde{h}_j = \sum_{j=k}^n v_{j,k} h_j, \nonumber  \\
	& = \sum_{j=k}^n v_{j,k} h_j - \sum_{j=k+1}^{m} \sum_{r=j}^n v_{r,k} v_{r,j} \tilde{h}_{j} , \nonumber\\
	& =\sum_{r=k}^n v_{r,k} h_r -  \sum_{r=k+1}^{n} v_{r,k} \sum_{j=k+1}^{r}  v_{r,j} \tilde{h}_j, \label{eq:93} \\
	& = U'_{*,k} (h - \sum_{j=k+1}^{m}  U_{*,j} \tilde{h}_j), \label{eq:94}
	\end{align}
	where the passage from Equation \eqref{eq:93} to \eqref{eq:94} is justified because $v_{r,j}=0$ for $j>r$. Therefore, $\sum_{j=k+1}^{r} v_{r,j}\tilde{h}_{j} = \sum_{j=k+1}^{m} v_{r,j}\tilde{h}_{j} $.
	
	\begin{table*}[t]
		\centering
		\begin{tabular}{cccc}
			\hline 
			&	Operation & Flop count & Total Flop count  \\
			&		& for iteration $k$ &   for $m$ iteration\\
			\hline 
			\multirow{2}{*}{FP}	&	$\tilde{h}_{k} \leftarrow \frac{2}{N_{k} } U_{*, k}' H_{*,k+1}$ & $2(n -k) +3$ &  \multirow{2}{*}{$(4n-m+2)m$} \\
			
			&	$H_{*,k} \leftarrow H_{*,k+1} - \tilde{h}_{k}  U_{*,k}$	& $ 2n $ & 	\\
			\hline 	
			\hline 
			\multirow{3}{*}{BP}	&	$\tilde{h}_{k} \leftarrow \frac{2}{N_{k} } U_{*, k}' H_{*,k+1}$ & $2 (n-k) + 3$ &  \multirow{3}{*}{$(7n -2m +3)m$}  \\
			&	$g\leftarrow g- \tilde{C}_{k}  U_{*,k}$	& $ 2(n-k+1) $ & 	\\
			&		$G_{*, k} \leftarrow - \tilde{h}_k g - \tilde{C}_k H_{*, k+1}$ & $3n$  & \\
			\hline 		
		\end{tabular}
		\caption{Time complexities of different operations in algorithm \ref{alg}. It is assumed that the matrix $U\in \mathbb{R}^{n\times m}$ is defined as in Equation \eqref{eq101}.}
		\label{TC-BP-alg}
	\end{table*}
	
	By setting $H_{*,k+1} \coloneqq h - \sum_{j=k+1}^{m}  U_{*,j}\tilde{h}_j$, and noting that $t_{k,k} = \frac{U'_{*,k} U_{*,k}}{2}$, we get 
	\begin{align}
	\tilde{h}_k &= \frac{2}{U'_{*,k} U_{*,k}}  U'_{*,k} H_{*, k+1}, \label{eq:91}\\
	H_{*, k} &= H_{*,k+1} - \tilde{h}_k U_{*,k}. \label{eq:92}
	\end{align}
	
	Equations \eqref{eq:91} and \eqref{eq:92} explain the lines 8 and 9 in Algorithm \ref{alg}. Note that $H_{*,1} = h - \sum_{j=1}^{m}  U_{*,j} \tilde{h}_{j} = h - \sum_{j=1}^m U_{*,j} [T^{-1}U' h]_j = h - U T^{-1} U' h = W h$. Hence, when $h = h^{(t-1)}$, we have $H_{*,1} = C^{(t)}$, which explains line 16 in Algorithm \ref{alg}.
	
	\textbf{Solving  the triangular system $T' \tilde{C} = U' \frac{\partial  \mathcal{L}}{\partial C}$}. Similarly to the previous case, we have for $1 \leq k \leq m$
	\begin{align}
	t_{k,k} \tilde{C}_k &+ \sum_{j=1}^{k-1} t_{j,k} \tilde{C}_j = \sum_{j=k}^n v_{j,k}\left[\frac{\partial \mathcal{L}}{\partial C}\right]_j, \nonumber \\
	& = \sum_{j=1}^n v_{j,k}\left[\frac{\partial \mathcal{L}}{\partial C}\right]_j -  \sum_{j=1}^{k-1}  \sum_{r=k}^{n}  v_{r,j} v_{r,k}  \tilde{C}_j, \label{eq:88} \\
	& = \sum_{r=1}^n v_{r,k}\left[\frac{ \partial\mathcal{L}}{\partial C}\right]_r - \sum_{r=1}^{n} v_{r,k} \sum_{j=1}^{k-1} v_{r,j}  \tilde{C}_j, \label{eq:89} \\ 
	& = U'_{*,k} \left(\frac{\partial \mathcal{L}}{\partial C}  -  \sum_{j=1}^{k-1}   U_{*, j} \tilde{C}_j\right), \nonumber
	\end{align}
	where the passage from Equation \eqref{eq:88} to \eqref{eq:89} is justified by the fact that $\sum_{r=k}^{n}  v_{r,j} v_{r,k}  \tilde{C}_j = \sum_{r=1}^{n}  v_{r,j} v_{r,k}  \tilde{C}_j$ (since $v_{r,k}=0$ for $r<k$).
	
	By setting $g \coloneqq \frac{\partial \mathcal{L}}{\partial C^{(t)}}  -  \sum_{j=1}^{k-1}  U_{*, j} \tilde{C}_j $, we can write $\tilde{C}_k = \frac{2}{U'_{*,k} U_{*,k}}U_{*,k}' g$ which explains the lines 12 and 13 in Algorithm \ref{aa}. Note also that after $m$-iterations in the backward propagation loop in Algorithm \ref{alg}, we have $g = \frac{\partial \mathcal{L}}{\partial C^{(t)}}  -  \sum_{j=1}^{m}   U_{*, j}\tilde{C}_j = \frac{\partial \mathcal{L}}{\partial C^{(t)}}  -  U \tilde{C}= \frac{\partial \mathcal{L}}{ \partial h^{(t-1)} }$. This explains line 17 of Algorithm \ref{alg}.
	
	Finally, note that from Equation \eqref{eq5}, we have for $1 \leq i \leq n$  and $1 \leq  k  \leq m$ 
	\begin{align*}
	\left[\frac{\partial \mathcal{L}}{\partial U}\right]_{i,k} = &-  \left[ \frac{\partial \mathcal{L}}{ \partial C}\right]_{i} \tilde{h}_k - h_i \tilde{C}_k + \\   & \sum_{j=1}^{m} v_{i,j}  \left( ((\tilde{h}\tilde{C}') \circ B' )_{j,k} +  ( (\tilde{C}\tilde{h}') \circ B )_{j,k} \right), \\
	= & -  \left[ \frac{\partial \mathcal{L}}{ \partial C}\right]_{i} \tilde{h}_k - h_i \tilde{C}_k  + \\  
	& \sum_{j=1}^{m} v_{i,j}  \left( \tilde{h}_j\tilde{C}_k \frac{\llbracket     k\leq j \rrbracket}{1 + \delta_{j,k}} +  \tilde{C}_j \tilde{h}_k \frac{ \llbracket     j\leq k \rrbracket}{1 + \delta_{j,k}} \right),  \\
	= &  -  \left[ \frac{\partial \mathcal{L}}{ \partial C}\right]_{i} \tilde{h}_k - h_i \tilde{C}_k  + \\
	& \sum_{j=1}^{m} v_{i,j}  \left( \tilde{h}_j\tilde{C}_k\llbracket     k< j \rrbracket + \tilde{C}_j  \tilde{h}_k \llbracket     j\leq k \rrbracket\right), \\
	= & \tilde{C}_k \left(\sum_{j=k+1}^{m} v_{i,j} \tilde{h}_j  - h_i\right) \\ & +  \tilde{h}_k \left(\sum_{j=1}^{k} v_{i,j} \tilde{C}_j -\left[ \frac{\partial \mathcal{L}}{ \partial C}\right]_{i}\right).
	\end{align*}
	
	Therefore, when $C = C^{(t)}$ and $h=h^{(t-1)}$ we have 
	\begin{align*}
	\left[\frac{\partial \mathcal{L}}{\partial U^{(t)}}\right]_{*,k} = - \tilde{C}_k H_{*, k+1} - \tilde{h}_k g,
	\end{align*}
	where $g = \frac{\partial \mathcal{L}}{\partial C^{(t)}}  -  \sum_{j=1}^{k-1} \tilde{C}_j  U_{*, j}$. This explains lines 14 and 18 of Algorithm \ref{alg}.
	
	\section{Time complexity}
	\label{C}	
	Table \ref{TC-BP-alg} shows the flop count for different operations in the local backward and forward propagation steps in Algorithm \ref{alg}.

	\section{Matlab implementation of Algorithm \ref{alg}}
	\label{D}
	\begin{figure}[h]
\begin{lstlisting}[mathescape=true]
% Inputs: U - matrix of reflection vectors
%		 h - hidden state at time-step t-1
%		 BPg - Grad of loss w.r.t C=Wh
% Outputs: g, G, C=Wh
[n, m] = size(U);
G=zeros(n, m); H = zeros(n, m+1);
N = zeros(m); h_tilde = zeros(m);
%  Zero-initialisation not required above!
H(:,m+1)=h; g=BPg;
%%--Forward propagation--%%	
for k =0:m-1  
	N(m-k) = U(:, m-k)' * U(:, m-k);
	h_tilde(m-k)=2 / N(m-k) * ... 
			U(:, m-k)' * H(:,m-k+1);
	H(:,m-k)=H(:,m-k+1) - ...
			h_tilde(m-k) * U(:,m-k);
end	
C = H(:,1)
%%--Backward propagation--%%
for k=1:m 
	c_tilde_k = 2*U(:,k)' * g / N(k);
	g = g - c_tilde_k * U(:,k);
	G(:, k)=-h_tilde(k) * g - ...
			c_tilde_k*H(:,k+1);
end
		\end{lstlisting}
		\caption{MATLAB code performing one-step FP and BP required to compute $C^{(t)}$, $	\frac{\partial \mathcal{L}}{\partial h^{(t-1)}}$ (variable \lstinline{g} is the code), and  $\frac{\partial \mathcal{L}}{\partial U^{(t)}}$ (variable \lstinline{G} is the code). The required inputs for the FP and BP are, respectively, the tuples $(U, h^{(t-1)})$ and $(U, C^{(t)}, \frac{\partial \mathcal{L}}{\partial C^{(t)}})$. Note that $\frac{\partial \mathcal{L}}{\partial C^{(t)}}$ is variable \lstinline{BPg} in the Matlab code.}
		\label{eee}
	\end{figure}

\end{appendix}
	
\end{document}